\documentclass[sn-mathphys-num]{sn-jnl}% Math and Physical Sciences Numbered Reference Style 
\usepackage{graphicx}%
\usepackage{multirow}%
\usepackage{amsmath,amssymb,amsfonts}%
\usepackage{amsthm}%
\usepackage{mathrsfs}%
\usepackage[title]{appendix}%
\usepackage{xcolor}%
\usepackage{textcomp}%
\usepackage{manyfoot}%
\usepackage{booktabs}%
\usepackage{algorithm}%
\usepackage{algorithmicx}%
\usepackage{algpseudocode}%
\usepackage{listings}%

\usepackage{bbding}
\usepackage{multirow}
\theoremstyle{thmstyleone}%
\newtheorem{theorem}{Theorem}%  meant for continuous numbers
\newtheorem{proposition}[theorem]{Proposition}% 
\newtheorem{assumption}[theorem]{assumption}% 

\theoremstyle{thmstyletwo}%

\theoremstyle{thmstylethree}%

\begin{document}

\title[Article Title]{Rethinking Information Loss in Medical Image Segmentation with
Various-sized Targets}

%%=============================================================%%
%% GivenName	-> \fnm{Joergen W.}
%% Particle	-> \spfx{van der} -> surname prefix
%% FamilyName	-> \sur{Ploeg}
%% Suffix	-> \sfx{IV}
%% \author*[1,2]{\fnm{Joergen W.} \spfx{van der} \sur{Ploeg} 
%%  \sfx{IV}}\email{iauthor@gmail.com}
%%=============================================================%%

\author[1,4]{\fnm{Tianyi} \sur{Liu}}\email{tianyi.liu2203@student.xjtlu.edu.cn}

\author[2,4]{\fnm{Zhaorui} \sur{Tan}}\email{zhaorui.tan21@student.xjtlu.edu.cn}
% \equalcont{These authors contributed equally to this work.}

\author[3]{\fnm{Kaizhu} \sur{Huang}}\email{kaizhu.huang@dukekunshan.edu.cn}
% \equalcont{These authors contributed equally to this work.}

\author*[1]{\fnm{Haochuan} \sur{Jiang}}\email{h.jiang@xjtlu.edu.cn}
% \equalcont{These authors contributed equally to this work.}

\affil*[1]{\orgdiv{School of Robotics, XJTLU Entrepreneur College (Taicang)}, \orgname{Xi’an Jiaotong-Liverpool University}, \orgaddress{\street{111 Taicang Road, Taicang}, \city{Suzhou}, \postcode{215123}, \state{Jiangsu}, \country{China}}}

\affil[2]{\orgdiv{School of Intelligent Science}, \orgname{Xi’an Jiaotong-Liverpool University}, \orgaddress{\street{111 Ren’ai Road, Suzhou Industrial Park}, \city{Suzhou}, \postcode{215123}, \state{Jiangsu}, \country{China}}}

\affil[3]{\orgdiv{Institute of Applied Physical Sciences
and Engineering}, \orgname{Duke Kunshan University}, \orgaddress{\street{No. 8 Duke Avenue}, \city{Suzhou}, \postcode{215316}, \state{Jiangsu}, \country{China}}}

\affil[4]{\orgdiv{School of Computer Science}, \orgname{University of Liverpool}, \orgaddress{\street{Brownlow Hill}, \city{Liverpool}, \postcode{L697ZX},\country{United Kingdom}}}

%%==================================%%
%% Sample for unstructured abstract %%
%%==================================%%

\abstract{Medical image segmentation presents the challenge of segmenting various-size targets, demanding the model to effectively capture both local and global information. Despite recent efforts using CNNs and ViTs to predict annotations of different scales, these approaches often struggle to effectively balance the detection of targets across varying sizes. Simply utilizing local information from CNNs and global relationships from ViTs without considering potential significant divergence in latent feature distributions may result in substantial information loss. To address this issue, in this paper, we will introduce a novel Stagger Network (SNet) and argues that a well-designed fusion structure can mitigate the divergence in latent feature distributions between CNNs and ViTs, thereby reducing information loss. Specifically, to emphasize both global dependencies and local focus, we design a Parallel Module to bridge the semantic gap. Meanwhile, we propose the Stagger Module, trying to fuse the selected features that are more semantically similar. An Information Recovery Module is further adopted to recover complementary information back to the network. As a key contribution, we theoretically analyze that the proposed parallel and stagger strategies would lead to less information loss, thus certifying the SNet's rationale. Experimental results clearly proved that the proposed SNet excels comparisons with recent SOTAs in segmenting on the Synapse dataset where targets are in various sizes. Besides, it also demonstrates superiority on the ACDC and the MoNuSeg datasets where targets are with more consistent dimensions.}


%%\pacs[JEL Classification]{D8, H51}

%%\pacs[MSC Classification]{35A01, 65L10, 65L12, 65L20, 65L70}

\maketitle

\section{Introduction}\label{sec:introduction}

Medical image segmentation has drawn  much attention from deep learning society~\cite{cai2020review,wang2023variable,su2023mind,yao2022novel}.
Accurate and generalized segmentation on various-sized targets, requiring capturing both local and global information for various-sized targets, will greatly assist radiologists in making treatment planning and post-treatment evaluations.

In the past few years, Convolutional Neural Networks (CNNs) have been widely used in medical image segmentation tasks. Focusing on local features~\cite{chang2021transclaw}, CNN-based models such as U-Net~\cite{falk2019u}, nnUNet~\cite{isensee2019nnu} and Res-UNet~\cite{diakogiannis2020resunet} are effective for smaller target predictions such as gallbladders and tissue aortas.
Despite successes achieved by CNN models in segmenting smaller targets, they are still restricted due to limited receptive fields and inherent inductive bias. With a weak ability to capture long-range dependency, CNNs still suffer from a lack of efficacy in predicting targets with relatively larger sizes, such as the livers and spleens.  
Vision Transformers (ViTs), on the other hand, are capable of learning long-range dependencies and capturing global contexts by using a multi-head self-attention mechanism.
% global contexts by using a multi-head self-attention mechanism~\cite{shi2023m}. 
They exhibit high precision in segmenting larger targets such as livers and spleens ~\cite{dosovitskiy2020image, valanarasu2021medical}. 
As for small targets, we statistically reveal (see Table~\ref{visual}) that although ViTs can segment certain them such as the kidneys, they fail to predict some other targets such as the gallbladders and tissue aortas. 

\begin{figure}[t!]
\centerline{\includegraphics[width=0.8\columnwidth]{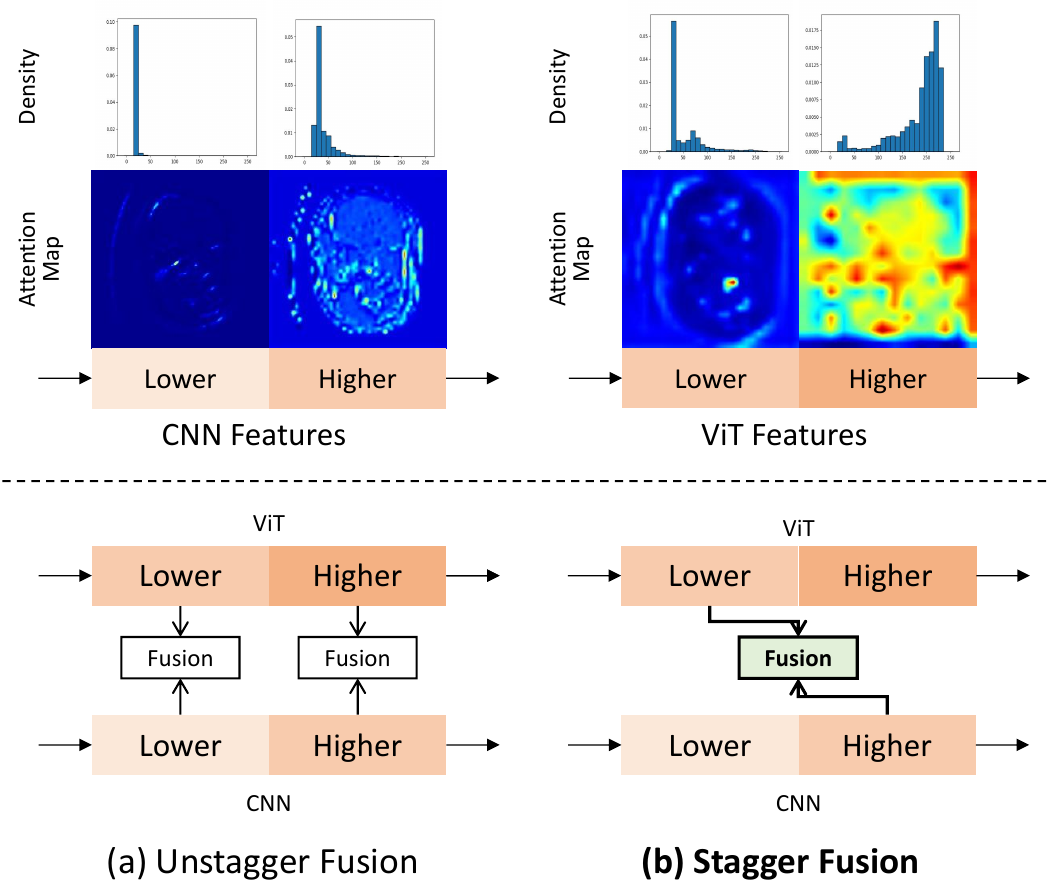}}
\caption{(Top) Visualization of feature heatmaps and histogram distributions of lower CNNs, higher CNNs, lower ViTs, and higher ViTs. Higher layers have a darker color than lower layers.
(Bottom) Unstagger fusion fuses lower layers of CNNs with those from lower ViTs, as well as features from higher layers of CNNs with those from higher ViTs. Stagger fusion fuses features from higher layers of CNNs and those from lower ViTs. Different colors represent dissimilar distributions of these feature layers.}
\label{representation}
\end{figure}

The aforementioned findings suggest that CNNs and ViTs offer complementary segmentation performance across larger and smaller targets.
Intuitively, latent features from both models can be fused to effectively predict targets of various sizes, expecting they can achieve simultaneous advantages.
Prior efforts in the literature such as TransAttUNet~\cite{chen2021transattunet}, Attention Upsample~\cite{li2021more}, and TransUnet~\cite{chen2021transunet} employ fusion modules to combine extracted features from both CNNs and ViTs. 
However, these arts fuse features in the unstagger manner, \textit{i.e.} features obtained from 
lower layers of CNNs and ViTs are fused (the same with higher-layer features, see Figure~\ref{representation}).
Our study reveals that these unstagger approaches overlook different modeling characteristics of each layer. This oversight can result in sub-optimal performance when segmenting targets with various sizes due to potential information loss. 
As shown in the bottom line in Figure~\ref{banner}~(a), attention maps of features from higher CNN and ViT layers appear distinctly. Higher CNN layers
focus on parts of the image, whereas higher ViT layers concentrate on more expanded regions. Moreover, most of the input attention focuses are weakened in the fused features; this suggests that information loss may possibly occur across two individual features, resulting in degraded segmentation. Similarly, fusion across features from lower CNN and ViT layers is also not ideal because they focus on distinct parts.

\begin{figure}[!t]
\centerline{\includegraphics[width=\columnwidth]{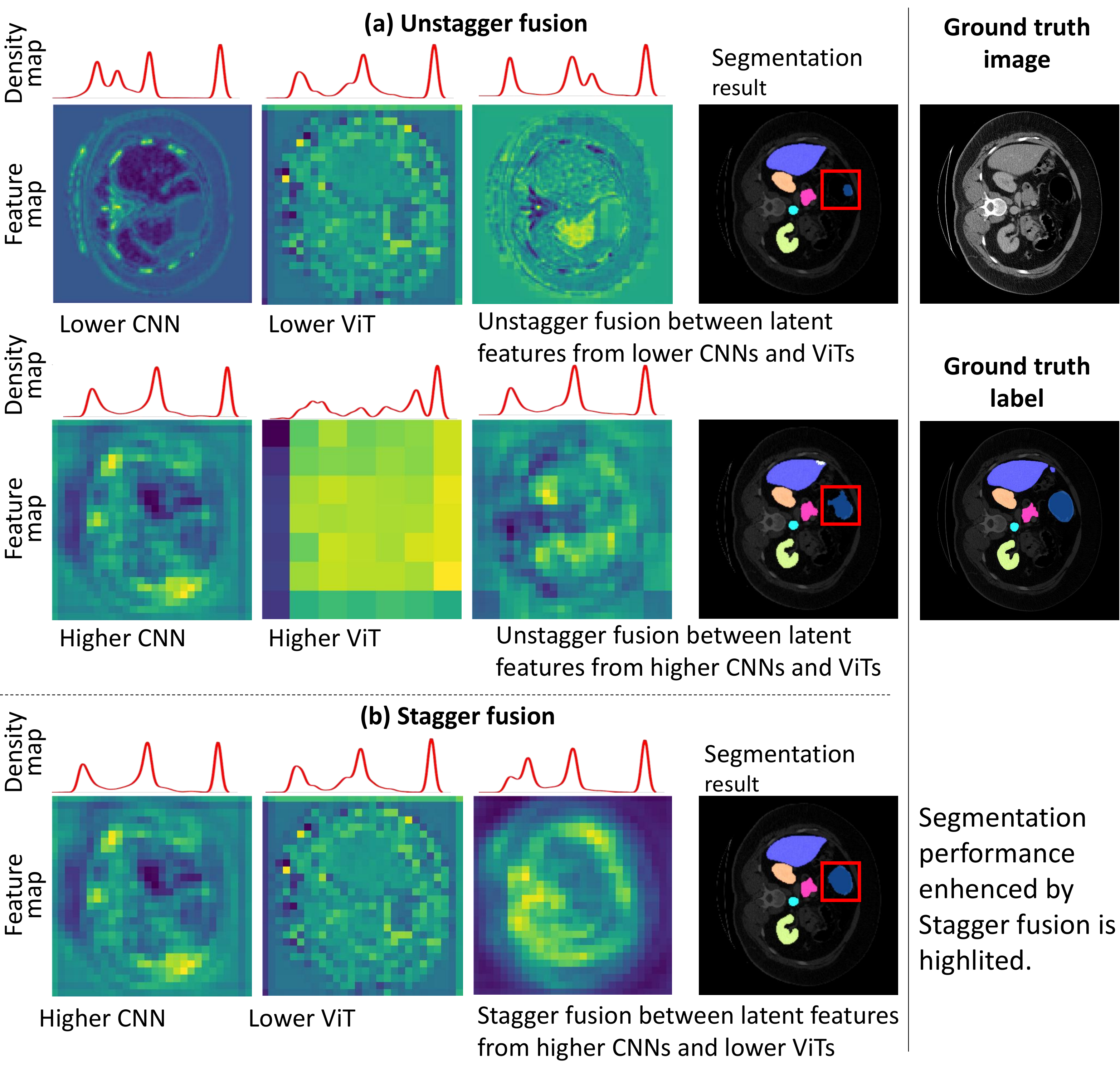}}
\caption{This figure depicts unstagger fusion in  (a) and stagger fusion in (b). 
Heatmaps visualizing input layers are presented on the first two columns, with the name of each layer located at the bottom and its corresponding density map situated above the heatmap. 
The heat maps and density maps of fusion results are illustrated in the third column. 
The segmentation results of each fusion method can be seen in the fourth column, and the input image and its ground-truth label can be seen in the fifth column. 
}
\label{banner}
\end{figure}
Drawing inspiration from the above observations, we argue that the information loss may be induced by large divergence across latent feature distributions. Empirically, as seen in the attention maps and density histograms in Figure~\ref{representation}, features from both CNNs and ViTs appear significantly different in the unstagger setting. Theoretically, we analyze that this non-negligible information loss is possibly brought by the unstagger fusion architecture (see Section~\ref{sec:inforloss} for details).

In this paper, we propose a novel model called Stagger Network (SNet) to tackle the information loss during feature fusion and promote segmentation performance across targets of various sizes. 
Specially,
it consists of three major modules: the Stagger Module, the Parallel Module, and the Information Recovery Module.
The Stagger Module with the feature fusion block achieves the core function to fuse the latent features from lower ViTs and higher CNNs in the stagger manner. With key theoretical evidence, we prove that this Stagger Module is more effective in reducing information loss in comparison with unstagger approaches. 
Additionally, we propose the Parallel Module with the feature enhancement block and the Information Recovery Module with the global attention block working as assisting components. 
In the Parallel Module, two series of consecutive enhanced features are produced by parallel CNN and ViT branches. The produced results will be sent to the Stagger Module. The Information Recovery Module, regarded as a feature decoder, further enhances fused information from the Stagger Module. 
Furthermore, as a unified network, the proposed SNet fuses the information from the CNN-based and ViT-based encoders and only employs the CNN-based decoder. 
It will save computational resources compared to using two distinct models to segment large and small targets separately. 

The major contributions of this paper are summarized as follows:
\begin{itemize}
    \item 
    We propose a novel Stagger Network with three modules: Parallel Module,  Stagger Module, and  Information Recovery Module. It can successfully segment both small and large medical imaging targets simultaneously.
    \item We theoretically show that the proposed Stagger Network combining higher CNNs and lower ViTs features will be superior to unstagger approaches, as it reduces information loss and promotes fusion efficacy.
    \item Extensive experiments demonstrate the effectiveness of our Stagger Network, not only showcasing significant improvements in predicting small targets but also ensuring high performance for larger targets. Specifically, SNet significantly improves the prediction score on both small targets by $9\%$ over SOTA on benchmarks Synapse~\cite{synapse}. It also outperforms over SOTA on ACDC~\cite{bernard2018deep} and MoNuSeg dataset~\cite{kumar2017dataset}. 
\end{itemize}

\section{Related Work}
\subsection{CNNs and ViTs}
CNNs in U-Net~\cite{ronneberger2015u} are particularly efficient in extracting local features in medical image segmentation. Transformers, on the other hand, excel at capturing long-range dependencies in sequences, though they are initially designed for language processing tasks~\cite{vaswani2017attention}. 
The first attempt to introduce transformers in vision tasks is known as the Vision Transformer (ViT)~\cite{dosovitskiy2020image}, achieving the state-of-the-art performance on the benchmark image classification dataset, the ImageNet~\cite{deng2009imagenet}.
Recent progress has also demonstrated successes with ViT variants in conventional computer vision (e.g., detection and segmentation) tasks, including DERT~\cite{carion2020end} and SegFormer~\cite{xie2021segformer}. TransUNet~\cite{chen2021transunet}, a ViT-based model, also shows its outstanding performance in the task of medical image segmentation.

\subsection{Feature Fusion Methods}
It is reasonable to combine CNNs with ViTs so that both strengths can be leveraged. In the following, we will give examples of typical cases that are promising to understand both local focus and long-range context.  
\subsubsection{Simple Replacement Methods}
One simple way to introduce ViTs in conventional CNNs is to 
replace some convolution Layers in a CNN model with some ViT blocks,
for example, TransClaw U-Net~\cite{chang2021transclaw}, Attention Upsample (AU)~\cite{li2021more}, Swin-Unet~\cite{cao2021swin} and TransAttUNet~\cite{chen2021transattunet}.
Concretely, in the TransClaw U-Net, ViTs are introduced in the higher encoding layers to replace CNNs; 
in the Attention Upsample (AU), window-based ViTs are placed in the decoding path, while the generating features are concatenated with encoding CNN features by the skip-connections; Swin-Unet replaces all CNNs with ViTs and constructs a pure ViT-based U-Net in that simple replacement of CNNs with ViTs cannot make full use of the advantages of CNN and Transformers~\cite{zhang2021transfuse}. The ability of CNN to locate low-level details may be lost when modeling global contexts.

\subsubsection{Advanced Fusion Methods}
There are also fusion proposals to explore the mutual relationship between the features generated by CNNs and ViTs. Typical examples can be found in Missformer~\cite{huang2021missformer} and Transfuse~\cite{zhang2021transfuse}.
In particular, an enhanced Transformer Context Bridge is employed in the Missformer~\cite{huang2021missformer} which introduces depth-wise CNNs in the Transformer blocks to model remote dependencies and local contexts.  
Although it fuses features that suit both CNNs and ViTs and demonstrates excellent performance in large targets, it does not present superiority in small targets, e.g. aorta, gallbladder, and kidneys,  empirically.  
Meanwhile, Transfuse~\cite{zhang2021transfuse} features two parallel ViT and CNN branches by feeding the same size features to the proposed BiFusion module in an unstagger fusion (see Figure~\ref{banner}) with a self-attention mechanism. 
However, it combines CNNs and ViTs unstaggeringly without considering the distinctive feature representation of each other, resulting in possible semantic gaps. 
On the contrary, our proposed SNet designs the stagger fusion strategy, promoting the fusion of features with similar representations, thus effectively alleviating information loss.

\section{Theoretical Motivation of Stagger Fusion}
\label{sec:inforloss}

Consider two discrete random variables from distributions $f^{a}\sim P^{a}, f^{b} \sim P^{b}$ as 
the latent features of CNN and ViT,
where $a$ and $b$ denote dimensions, and $P^a$ and $P^b$ denote their distributions respectively. 
The joint entropy is determined by the marginal distributions of multiple random variables and their joint distribution. Minimizing joint entropy involves finding a joint distribution that enhances the certainty of relationships among variables. In this paper, we endeavor to minimize the joint entropy of CNN and ViT, thereby reducing the uncertainty of these two joint distributions and mitigating information loss during the fusion process.
To achieve better fusion between $f^a$ and $f^b$, we set the optimization objective toward a lower joint entropy $H(f^a,f^b)$ between them:
\begin{equation}
\label{entropy}
H(f^a,f^b) = H(f^a) + H(f^b) - I(f^a;f^b),
\end{equation}
where  $H(f^a)$ and $H(f^b)$ are entropy of $f^a$ and $f^b$, and $I(f^a;f^b)$ is the mutual information of $f^a$ and $f^b$,
given that $H(f^a)$ and $H(f^b)$ remain relatively stable \textit{w.r.t.} $f^a$ and $f^b$, the primary objective becomes maximizing $I(f^a;f^b)$.

However, the blend of latent features from the lower CNN and ViT layers or from higher CNN and ViT layers may decrease $I(f^a;f^b)$.
As seen from Figure~\ref{representation}~(a), lower CNN layers pay more attention to local parts, whilst lower ViT layers will focus more on global representations. 
Previous work~\cite{raghu2021vision} also shows lower CNN (e.g. Resnet) and ViT (e.g. ViT L/16) features have large feature distribution divergences. 
As such, combining the lower features of both CNNs and ViTs would magnify $H(f^a,f^b)$, resulting in loss of information.
So do the feature distributions of higher ViTs and CNNs. We provide a theoretical analysis as follows.

\begin{assumption}
\label{ass:dim}
We denote $f^{n^*}\sim P^{n^*}$ as an optimal fused feature between $f^{a}$ and $f^{b}$  with $n^*$-dimensions, where $n^* \in [\max(a,b), a+b]$; $a, b$ denote respective dimensions and $P^a$ and $P^b$ denotes their distributions. 
% (\textbf{It seems to be a fragment of sentences. Please double check this---by KH})
\end{assumption}

Assumption~\ref{ass:dim} holds because $n^* = a+b$ \textit{iff} $P^{a}$ and $P^{b}$ are absolutely independent of each other; $n^* = \max(a,b)$ \textit{iff} one of $P^{a}$ and $P^{b}$ is fully dependent on the other one.
Consider a fusion operation  $\mathcal{F}$ that can maintain all information in $f^a, f^b$.
According to  \textit{Jensen's inequality}~\cite{jensen1906fonctions}, we have:
\begin{align}
\label{eq:H_ineq}
    H(f^{n^*}) = H(\mathcal{F}(f^a, f^b)) \leq H(f^a, f^b),
\end{align}
where $H(f^a, f^b)$ can be considered as an upper bound of $H(f^{n^*})$. 

However, the absolute optimal solution $f^{n^*}$ is hard to be obtained. 
In the fusion model setting, it pursues sub-optimal solutions, generating the fused feature $f^n \sim P^n$  dimensionalized by $n$.
Be noted that $n$ is pre-defined by the model structure as a hyper-parameter. 
In common scenarios, setting $n > \max(a,b)$ is necessary to avoid possible information loss.
Setting $n < a+b$ because they are easy to correlate to some extent since $f^a$ and $f^b$ are the features extracted from the same input image.
Since Eq.~\ref{eq:H_ineq} only holds when $\mathcal{F}$ bring no information loss,
we identify that when the model structure is fixed, finding a proper latent layer for $\mathcal{F}$
that can obtain the lower $H(f^a, f^b)$ is crucial for the fusion operation.

Before we propose our main proposition, we first elaborate \textit{Han's inequality}:
\begin{theorem}[Han's inequality~\cite{boucheronconcentration}] 
\label{hans}
The Han's inequality is presented below: Let $X^i$ be discrete $i$-dimensional random variable and denote $\bar{H}^k\left(X^i\right)=\frac{1}{\left(^i_k \right)} \sum_{T \subset \left(^{[i]}_{\;k} \right)} H(X_{T})$
as the average entropy of randomly selected $k$ dimensions $(k\leq i)$. Then $\frac{1}{k} \bar{H}^k$ is decreasing in $k$: 
\begin{align}
\label{eq:han}
    \frac{1}{i} \bar{H}^i \leq \cdots \leq \frac{1}{k} \bar{H}^k \cdots \leq \bar{H}^1.
\end{align}
\end{theorem}
Eq.~\ref{eq:han} indicates that the mean entropy on each dimension decreases as the number of $k$ increases. 
Based on \textit{Han's inequality}, we have the Proposition~\ref{prop:fusion}:
\begin{proposition}
\label{prop:fusion}
     When $n$ is fixed, i.e., the fusion model structure is fixed and Assumption~\ref{ass:dim} holds, the information loss depends on what $f^a$ and $f^b$ from model latent layers are selected. Specifically,
     % If the distributions of the two features are very different,
     information can be lost when the divergence between distributions of $f^a$ and $f^b$ is large.
\end{proposition}

\begin{proof}
    We denote $\bar{H}^n\left(f^n\right), \bar{H}^{n^*}\left(f^{n^*}\right)$ as the 
    average entropy of all dimensions of $f^n, f^{n^*}$, respectively.
    We discuss three situations:

     1). If dimensions in $P^{a}$ and $P^{b}$ are most independent, it has
    $n < n^* \leq a+b$ and $ n^* = a+b$ when $P^{a}$ and $P^{b}$ are fully independent from each other.
    % the divergence between $F_{a}$ and $F_{b}$ is large or 
    % $P^{a}$ and $P^{b}$ are independent. Following Theorem~\ref{hans}, we have:
    Under this circumstance, it has:
    \begin{align}\label{eq:2}
        \frac{1}{n^*} \bar{H}^{n^*} < \frac{1}{n} \bar{H}^{n}, 
    \end{align}
    where information will be likely to be lost.

     2). If proper $f^a$ and $f^b$ are used and $n \approx n^*$, it approaches the optimal fusion solution without information loss:        
    \begin{align}\label{eq:3}
        \frac{1}{n^*} \bar{H}^{n^*} \approx \frac{1}{n} \bar{H}^{n}{.}
        \end{align}

    3). If $P^{a}$, $P^b$ largely depend on each other, it has $\max(a,b) \leq n^* <n$ and $\max(a,b) = n^*$ \textit{iff}  $P^{a}$ ($P^b$) is fully depends on $P^{b}$ ($P^{a}$) or vice versa. In this scenario, following Theorem~\ref{hans}, the following holds: 
    \begin{align}\label{eq:1}
        \frac{1}{n^*} \bar{H}^{n^*} > \frac{1}{n} \bar{H}^{n},
    \end{align}
    where information will be unlikely to be lost.

      {Therefore Proposition~\ref{prop:fusion} holds.}
   
    \end{proof}

Proposition~\ref{prop:fusion} indicates that when the model structure is fixed, finding latent feature space for conducting fusion methods is critical. Our experiments reveal that using unstagger fusion proposed in previous methods tends to be under the scenario in Eq.~\ref{eq:2} since there are significant differences in distributions, leading to sub-optimal results. To tackle this problem, we propose Stagger Module (in Sec.~\ref{stager}) to ensure that the fusion meets the scenarios where Eq.~\ref{eq:3} and Eq.~\ref{eq:1} hold. Furthermore, extensive experiments validate that our Stagger Module performs efficient fusion and outperforms the previous methods.

\begin{figure}[!t]
\centerline{\includegraphics[width=\columnwidth]{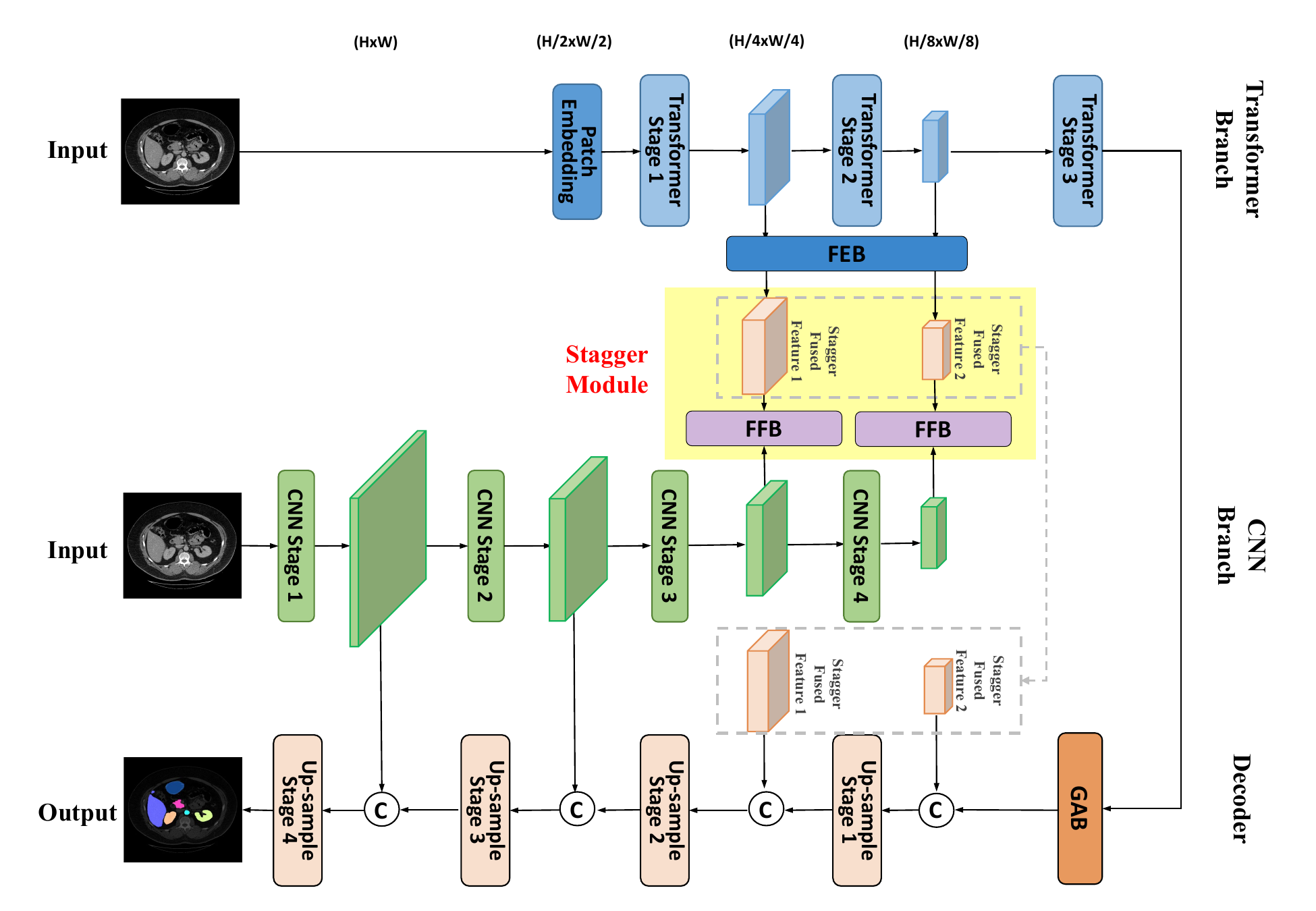}}
\caption{The overall framework of SNet. The label C in a circle means concatenate.
}
\label{are}
\end{figure}

\section{Methodology}
\subsection{Overview}
In this section, we will introduce the architecture of the proposed SNet with the Parallel Module in~Sec.~\ref{sec:pp}, Stagger Module in~Sec.~\ref{sec:sp}, and Information Recovery Module in~Sec.~\ref{sec:irp}. The overall architecture can be seen in Figure~\ref{are}. 
In the Parallel Module, there are two branches with Feature Enhancement Block (FEB), \textit{i.e.,} CNN and ViT branches.
They generate two sets of features, where each set is made up of features generated by consecutive CNNs or ViTs in distinct branches.
As mentioned in Sec.~\ref{sec:inforloss}, decreasing the information loss is the main objective. Therefore, the Stagger Module is the main module. In this module, lower ViTs and higher CNNs will be fused in the Feature Fusion Block (FFB). The fused information is enhanced in the Global Attention Block (GAB) in the Information Recovery Module. This module functions as the decoder, thereby completing the entire U-Net structure.

\subsection{Parallel Module}
\label{sec:pp}
The proposed SNet consists of two parallel branches: a CNN branch and a ViT branch. It exploits and optimizes inherent advantages from both architectures, as well as achieving comprehensive feature representations. 
% Besides, it also effectively avoids the early fusion that would potentially bring information loss.
To be specific, at the start, the input  will be sent to both the CNN and ViT branches.
In the CNN branch, features will be convoluted and down-sampled serially. 
Differently, in the ViT branch, features are extracted by vision transformers, before being down-sampled by patch embedding~\cite{dosovitskiy2020image}.
In the Parallel Module, we employ the Feature Enhancement Block (FEB) to decrease entropy of ViT latent features, $H(f^b)$. It ensures more concise and information-rich feature representations and helps decrease the joint entropy $H(f^a,f^b)$ (Eq.~\ref{entropy}).

\begin{figure}[]
\centerline{\includegraphics[width=0.55\columnwidth]{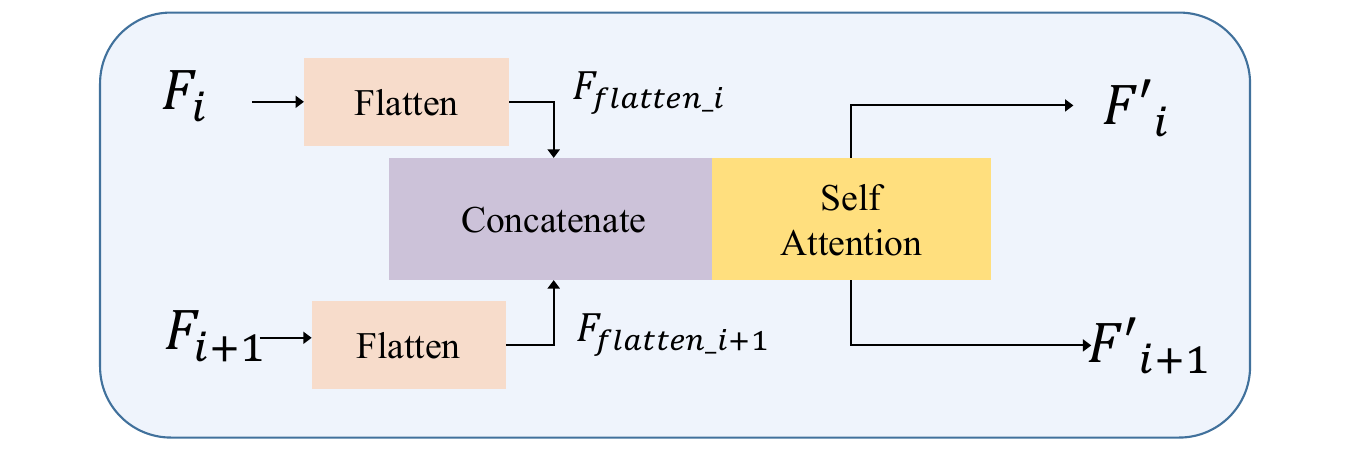}}
\caption{Feature Enhancement Module (FEB): As seen in Figure~\ref{are}, the raw image is the input of the 1st ViT layer. After the down-sampling, the output of the 1st ViT layer becomes the input of the 2nd ViT layer. Then the output of the 1st and 2nd ViT layer $F_i$ and $F_{i+1}$ will be fused in FEB and then split back to two features $F^{'}_{i}$ and $F^{'}_{i+1}$. }
\label{FEB}
\end{figure}

\textbf{Feature Enhancement Block (FEB):}
FEB collects features from two consecutive ViT layers $F_i$ and $F_{i+1}$, and they will be flattened to vectors 
$F_{Flatten_i}$ $\in$ $\mathbb{R}^{N \times C}$, where
$C$ is the number of channels, $N_{i} = \frac{HW}{2^{2i}}$, and $i$ denotes the $i$-th encoding layer. 
% Thus in Figure~\ref{are} (b), the inputs $F_i$ and $F_{i+1}$ are two consecutive ViT layers and the outputs $F_i$ and $F_{i+1}$ are the features being enhanced.
As shown in Figure~\ref{FEB},  $F_{Flatten_i}$ and $F_{Flatten_{i+1}}$ are concatenated before being fed into the self-attention block to calculate interactions between each other and maximize the global representation in the FEB. 
After that, the produced vectors will be unflattened back to two feature maps $F^{'}_{i}$ and $F^{'}_{i+1}$ with identical dimensions of $F_{i}$ and $F_{i+1}$ before they are sent to two individual Depth-wise CNNs (DW-Conv) to prevent losing local relationships. The enhanced $F_i$ and $F_{i+1}$ will be fused with features from CNN branch in the Stagger Module, discussed in Sec.~\ref{stager}.

\subsection{Stagger Module}
\label{stager}
As mentioned in Sec.~\ref{sec:inforloss}, in the Stagger Module, features from lower ViTs and higher CNNs in the Parallel Module will be fused by using the stagger fusion method to minimize information loss. 
Features from lower ViT and higher CNN layers are selected ($F_C$ and $F_T$) to be fused by calculating when the calculated KL divergence is lower, indicating they are similar in distributions and less information loss (shown in Figure~\ref{representation} and discussed in Sec.~\ref{sec:inforloss}).

\begin{figure}[]
\centerline{\includegraphics[width=0.55\columnwidth]{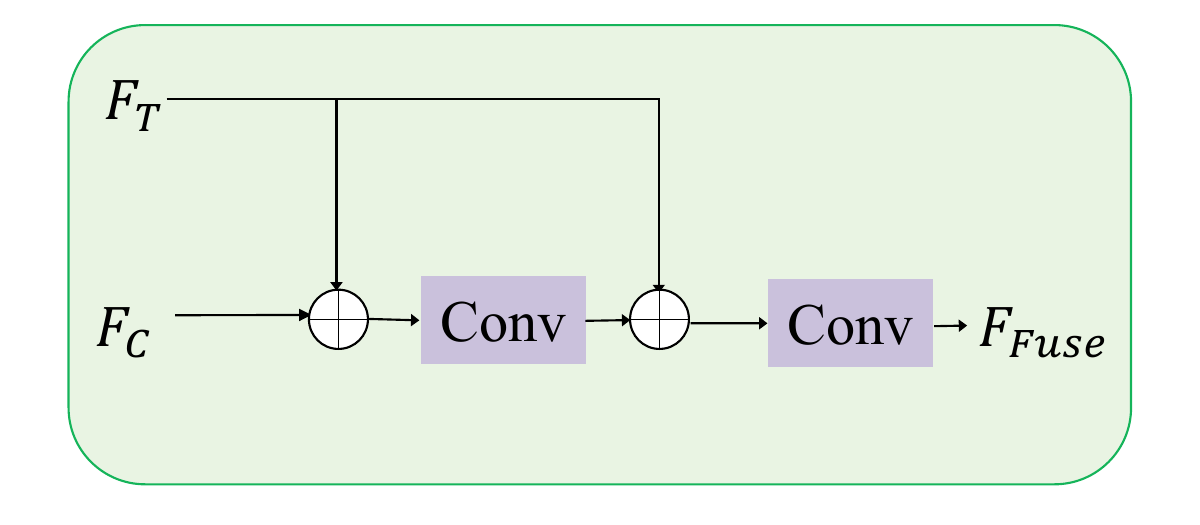}}
\caption{Feature Fusion Block (FFB): $\bigoplus$ means concatenation.}
\label{FFB}
\end{figure}

\label{sec:sp}
\textbf{Feature Fusion Block (FFB):}
As seen in Figure~\ref{FFB}, the proposed FFB blocks receive fused features extracted by higher CNN layers represented as $F_C$ and lower ViT layers represented as $F_T$ from the Parallel Module.  
The number of channels of ViT features sent to the FFB is four times smaller than that of CNN features. 
If the fusion is implemented by simple concatenation, contributions of CNN features relative to ViT features will be improved, potentially resulting in an excessive reliance on CNN information.
The fusion will then tend to favor features from CNNs and neglect information from ViTs. 

To address it, in the proposed FFB, CNN ($F_C \in \mathbb{R}^{HW \times 4d}$) and ViT ($F_{T}\in \mathbb{R}^{HW{\times}{d}}$) features will be firstly concatenated to produce an initial fused feature  
$ F_{1}\in \mathbb{R}^{HW{\times}{5d}}$, given that $d$ is the number of channels of input CNN features. 
After that, a DW-Conv with batch normalization and GeLU nonlinearity (Conv-BN-GeLU) will be applied to this concatenated map, producing 
% $F_2 \in \mathbb{R}^{HW \times 2d}\label{eq}$\end{equation}$
$F_2 \in \mathbb{R}^{HW \times 2d}$. 
Then, $F_2$ will be concatenated with $F_C$ again before sent to another Conv-BN-GeLU again to produce 
$F_{Fuse} \in \mathbb{R}^{HW \times 2d}$.
% \end{equation}
In this sense, features from both CNN and ViT will be mostly balanced. At the same time, the dimension of the fused feature is in the range of $d, 5d$, as we stated in Assumption~\ref{ass:dim}. 
After that, the fused feature will be sent to the information recovery to be part of the decoder inputs.

\subsection{Information Recovery Module}
\label{sec:irp}
In contrast to the fusion modules used in the Parallel and Stagger Modules, we use CNN blocks in the Information Recovery Module to up-sample the
extracted deep features. The up-sampling operation similar to the U-Net decoder reshapes the feature maps of adjacent dimensions into a higher-resolution feature map and reduces the feature dimension to half of the original dimension accordingly.

\textbf{Global Attention Block (GAB):}
\begin{figure}[]
\centerline{\includegraphics[width=0.55\columnwidth]{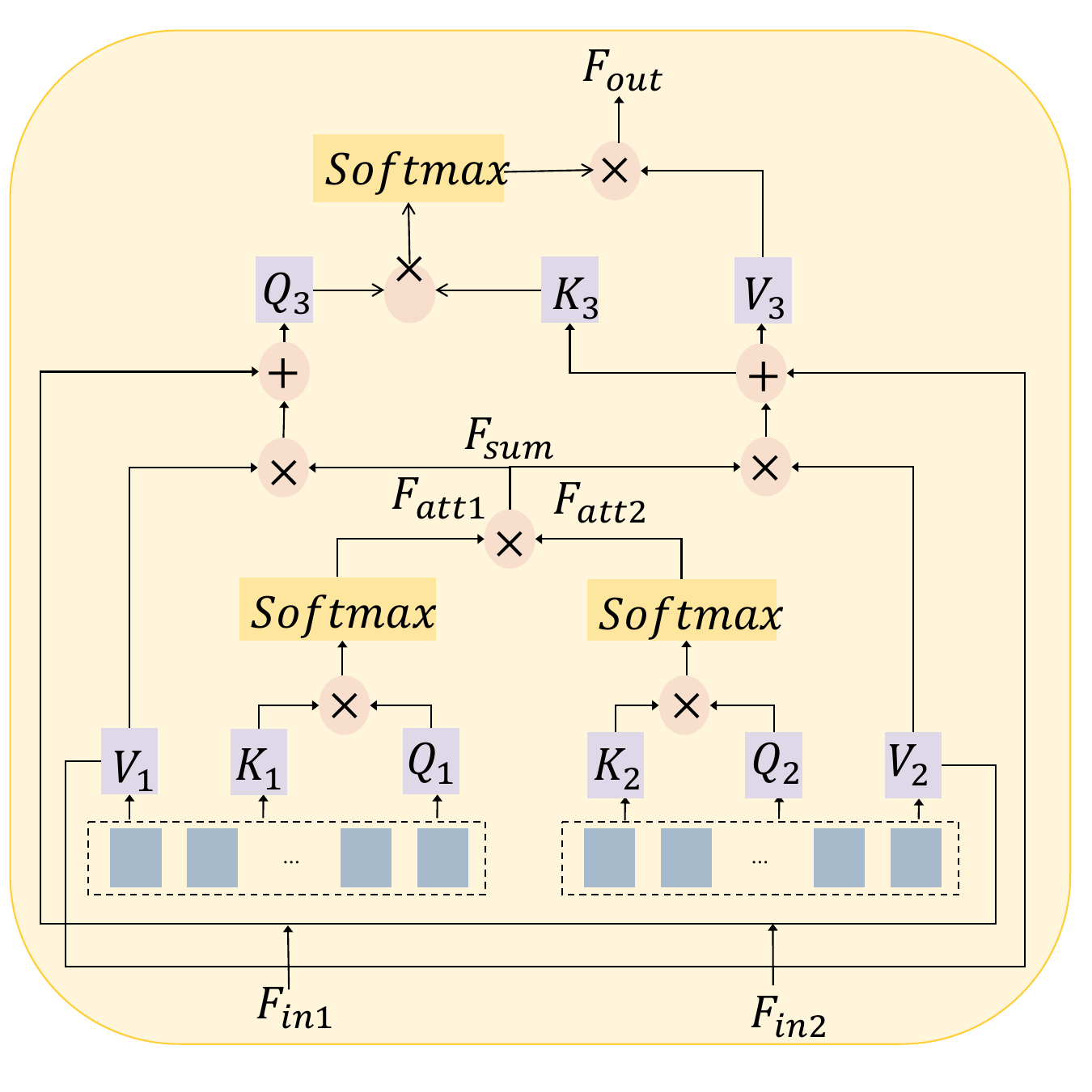}}
\caption{Global Attention Block (GAB): The input of GAM is the third ViT layer and the fourth ViT layer ($F_{in1}$ from $L_1$ and $F_{in2}$ from $L_2$ respectively) which is simplified in Figure~\ref{are}.}
\label{GAM}
\end{figure}
To recover the information from higher ViT layers that have not been processed before, a novel Global Attention Block (GAB), based on the idea from~\cite{aghdam2022attention}, will be engaged in the proposed SNet. It is essentially a two-layer attention mechanism.
Queries ($Q_1$ and $Q_2$ respectively for $L_1$ and $L_2$), keys ($K_1$ and $K_2$), and values ($V_1$ and $V_2$) are obtained by $Q_1=V_1=K_1=F_{in1}$ and $Q_2=V_2=K_2=F_{in2}$. 
Then, a new attention map will be calculated by 
\begin{equation}
F_{sum} = softmax(\frac{Q_1 K_1^T}{\sqrt{d}}) + softmax(\frac{Q_2 K_2^T}{\sqrt{d}}) {,}
\end{equation}
to \textit{exaggerate} the local importance across features from two layers while modeling the interaction between them at the same time.
After this, values from two layers, namely $V_1$ and $V_2$, will be swapped to produce 
 two layers via 
 \begin{equation}
F_1 = W_{sum} \times V_1 + V_2 {,}
\end{equation} 
% and 
\begin{equation}
     F_2=W_{sum} \times V_2 + V_1 {,}
 \end{equation} to further guide the interaction between them.  
 GAB's final result  will be obtained by the cross-attention mechanism via 
\begin{equation}
    F_{out} =  softmax(\frac{QK^T}{\sqrt{d}})V {.}
\end{equation} 
Here $Q = F_1W_q$, $K=F_2W_k$, and $V=F_2W_v$. $W_q$, $W_k$, and $W_v$ are trainable matrices. It enhances the global relevance from features generated in consecutive ViT layers.

\section{Experiments}
% Please add one paragraph to describe the overall experiments. \textbf{---Added by KH}
In this section, the experimental settings are first detailed. Experiment results on three datasets Synapse, ADCD and MoNuSeg will then be presented to demonstrate the effectiveness of the SNet. For fair comparisons, we adopt several SOTA models for better evaluation of the proposed SNet. We also conducted several ablation studies to verify the necessity of each component mentioned in the model.

\subsection{Experimental Setup}
\subsubsection{Datasets}\label{sec:Synapes}
Synapse~\cite{synapse} consists of 30 3D Computed Tomography (CT) scan subjects to segment 13 abdominal organs. Following SwinUnet~\cite{cao2021swin} and TransUnet~\cite{chen2021transunet}, we select 8 annotations, i.e., aorta, gallbladder, spleen, left kidney, right kidney, liver, pancreas, and stomach.
It is noted that we regard the spleen, liver, and stomach as larger organs, the aorta is seen as vascular tissues, and the remaining four organs are shaped relatively smaller~\cite{unetr}.                           
Splits of training and testing sets are also formed by SwinUnet~\cite{cao2021swin} and TransUnet~\cite{chen2021transunet}. The Average Dice-Similarity Coefficient (dice score) and the Hausdorff Distance (HD) are employed to evaluate the model performance.

The ACDC dataset consists of 100 3D cardiac Magnetic Resonance Imaging (MRI) subjects with annotations including the Right Ventricle (RV), Myocardium (Myo), and Left Ventricle (LV). 
Splits of training and testing sets are also formed by SwinUnet~\cite{cao2021swin} and TransUnet~\cite{chen2021transunet}. The Average Dice-Similarity coefficient is employed to evaluate the model performance. 

% \red{Glas contains a total of 165 images, which are histopathological sections of colon cancer stained with H&E from 16 different patients. According to relevant literature, 85 images in this dataset were used for training the network, and the remaining 80 were used for testing the network.}

{MoNuSeg contains 44 images
which are tissue images from different patients and organs and magnified 40 times. The dataset includes approximately 29,000 nuclear boundary annotations. According
to relevant literature, 30 images in this dataset were used for training
the network, and the remaining 14 were used for testing the network. Dice score and IoU are used to evaluate the model performance according to CT-Net~\cite{ZHANG2024298}.}

\subsubsection{Settings}
Our proposed SNet is trained for 300 epochs for the Synapse dataset and 150 epochs for the ACDC dataset on NVIDIA 3080Ti with 12 GB memory based on Pytorch 1.10. No pre-trained weights are employed. During the training, the batch size is set to 12, and the SGD optimizer with momentum 0.9 and weight decay 1e-4 is used. Following~\cite{cao2021swin,chen2021transunet}, we clip the values in the Synapse data  to $[-125, 275]$ which are then normalized to $[0,1]$. At this stage, we treat each slice in 3D subjects as one individual 2D image, and they will be spatially resized to $224\times224$. Common data augmentation techniques including flips and rotations are used to promote data diversity and model robustness. No pre-trained model is used for training.

\subsubsection{Training Strategy}
As seen from Figure~\ref{are}, the final segmentation result $\hat{y}$ will be supervised by optimizing the binary cross entropy ($BCE$) and the dice ($Dice$) losses referring to true annotations. 
To further guide fused feature representation, we leverage the deep supervision  strategy~\cite{zhou2018unet++} on the fused features $\hat{y_{f}}$ by both the losses. 
The final objective to optimize the proposed SNet is given by: $\ell=0.6BEC(\hat y_{f},y)+0.4Dice(\hat y_{f},y)+0.6BCE(\hat y,y)+0.4Dice(\hat y,y)$, where $y$ is the ground-truth.

\subsection{Experimental Results} 

\subsubsection{Results on Synapse Dataset}
In Table~\ref{table:1}, we present the results of the proposed SNet against several state-of-the-art baselines on the Synapse dataset. The proposed SNet achieves the highest dice score and lowest HD scores on average segmentation performance across 8 organs selected. 
When we look into each of the individual organ predictions, we found that improvements brought by SNet are more significant on three organ segmentations including Gallbladder and Kidney (L and R) by $2.13\%$, $2.42\%$, and $1.92\%$ dice score respectively. It also improves the segmentation performance on the Aorta by over $1\%$ in dice score.
Meanwhile, when segmenting larger organs such as the liver and stomach, SNet can still promote the prediction performance by $0.1\%$ and $0.03\% $ in dice score respectively. 
SNet and Missformer obtains comparable performance on spleen segmentation with a $1.93\%$ dice score difference.

From Table~\ref{table:1}, it can be seen that the improvement effect of SNet on large object segmentation is not significant, such as in the Liver and stomach, and even inferior to SwinUnet, CASTformers, and MISSformer on Spleen. However, the descent on the large object is relatively subtle, with a more pronounced improvement observed in smaller objectives. The enhancement of smaller objectives significantly outweighs the decline of larger objectives. 
Although SNet only surpasses CASTformers by a mere 0.5\% dice scores, it also has improvement on small targets.
This is due to a trade-off between general goals and details within the model. It is reasonable to expect some decrease in performance on larger objectives while there is an ascent in performance on smaller objectives. Nevertheless, the overall effect is positive, indicating an improvement in the model's performance.

\begin{sidewaystable}
% \scriptsize
\tiny
% \small
\centering
\caption{
Comparison of SNet and other advanced methods on the Synapse dataset.
Bold indicates the best result, and underline indicates the second best. 
% The results of relevant experiments are original from Missformer. The experimental setups of SNet are the same as that of Missformer, such as the test cases, data preprocessing method, etc.
SNet is our model and SNet* is the proposed model with different hyper-parameters. Avg. is the average dice of all the classes.}
% \resizebox{\linewidth}{!}{%
\begin{tabular*}{\textwidth}{@{\extracolsep\fill}l|c|c|cccc|ccc|c}
\hline
\multirow{3}{*}{Methods}  &\multirow{3}{*}{HD $\downarrow$} & \multicolumn{9}{c}{\rule{0pt}{7pt}Dice ($\%$) $\uparrow$}  \\  
\cline{3-11}   
& & \multicolumn{1}{c|}{Tissue} & \multicolumn{4}{c|}{Small Target}& \multicolumn{3}{c|}{Large Target}&
\multirow{2}{*}{Avg.}\\
% \cline{3-10} 
% \rule{0pt}{8pt}
        &   & Aorta          & Gallbladder    & Kidney(L)      & Kidney(R)      & Pancreas          & Liver       & Spleen         & Stomach       \\ \hline
V-Net\cite{milletari2016v}               & -              & 75.34          & 51.87          & 77.10          & 80.75          & 40.05          & 87.84         & 80.56          & 56.98     & 68.81     \\
DARR\cite{fu2020domain}               & -              & 74.74          & 53.77          & 72.31          & 73.24          & 54.18           & 94.08          & 89.90          & 45.96      & 69.77    \\
U-Net\cite{ronneberger2015u}              & 39.70          & 89.07          & 69.72          & 77.77          & 68.60          & 53.98          & 93.43          & 86.67          & 75.58        & 76.85   \\
Att-UNet\cite{oktay2018attention}            & 36.02          & 89.55          & 68.88          & 77.98          & 71.11          & 58.04          & 93.57          & 87.30          & 75.75        & 77.77  \\ \hline
R50 ViT\cite{valanarasu2021medical}            & 32.87          & 73.73          & 55.13          & 75.80          & 72.20          & 45.99           & 91.51          & 81.99          & 73.95    & 71.29       \\
TransUNet\cite{chen2021transunet}           & 31.69          & 87.23          & 63.13          & 81.87          & 77.02          & 55.86          & 94.08           & 85.08          & 75.62    & 77.48      \\
TransClaw\cite{chang2021transclaw}         & 26.38          & 85.87          & 61.38          & 84.83          & 79.36          & 57.65          & 94.28         & 87.74          & 73.55      & 78.09      \\
SwinUnet\cite{cao2021swin}            & 21.55          & 85.47          & 66.53          & 83.28          & 79.61          & 56.58          & 94.29           & 90.66          & 76.60         & 79.13 \\ 
CASTformers\cite{you2022class}  &22.76 & 89.05 & 67.48 & 86.05 & 82.17 & \textbf{67.49} & {95.61}  & \underline{91.00} &  {81.55} & 82.55\\ 
MISSformer\cite{huang2021missformer} & \underline{18.20} & 86.99 & 68.65 & 85.21 & 82.00 & 65.67& 94.41 & \textbf{91.92} & 80.81  & 81.96\\ 
CTC–Net\cite{yuan2023effective} &22.52 & 86.46 & 63.53& 83.71& 80.79 &59.73& 93.78& 86.87& 72.39&78.41\\
HiFormer\cite{heidari2023hiformer}&19.14&87.03&68.61 &   84.23& 78.37& 60.77&94.07&  90.44 &\textbf{82.03}& 80.69\\    
CT-Net\cite{ZHANG2024298}& - & 89.00& 67.70& 84.10 &80.60&  \underline{67.90}& \textbf{96.20} & {90.00} & \textbf{85.00}& 82.60\\
\hline
SNet* (ours)    & 18.85 & \underline{89.58} & \textbf{71.01} & \textbf{88.47} & \underline{84.06} & 63.80 & \underline{95.71} & 88.80 & {81.58} &\underline{82.88}\\
% \hline
\textbf{SNet (ours)}   & \textbf{15.74} & \textbf{90.19} & \underline{69.48} &\underline{87.48} & \textbf{84.09} & {66.99} & 95.58 & 89.99 & 80.61& \textbf{83.05} \\\hline
\end{tabular*}
% }
\label{table:1}
% \caption{
% Comparison of SNet and other advanced methods on the Synapse dataset.
% Bold indicates the best result, and underline indicates the second best. 
% % The results of relevant experiments are original from Missformer. The experimental setups of SNet are the same as that of Missformer, such as the test cases, data preprocessing method, etc.
% SNet is our model and SNet* is the proposed model with different hyper-parameters. Avg. is the average dice of all the classes.}
\end{sidewaystable}

\begin{figure}[h!]
\centerline{\includegraphics[width=\columnwidth]{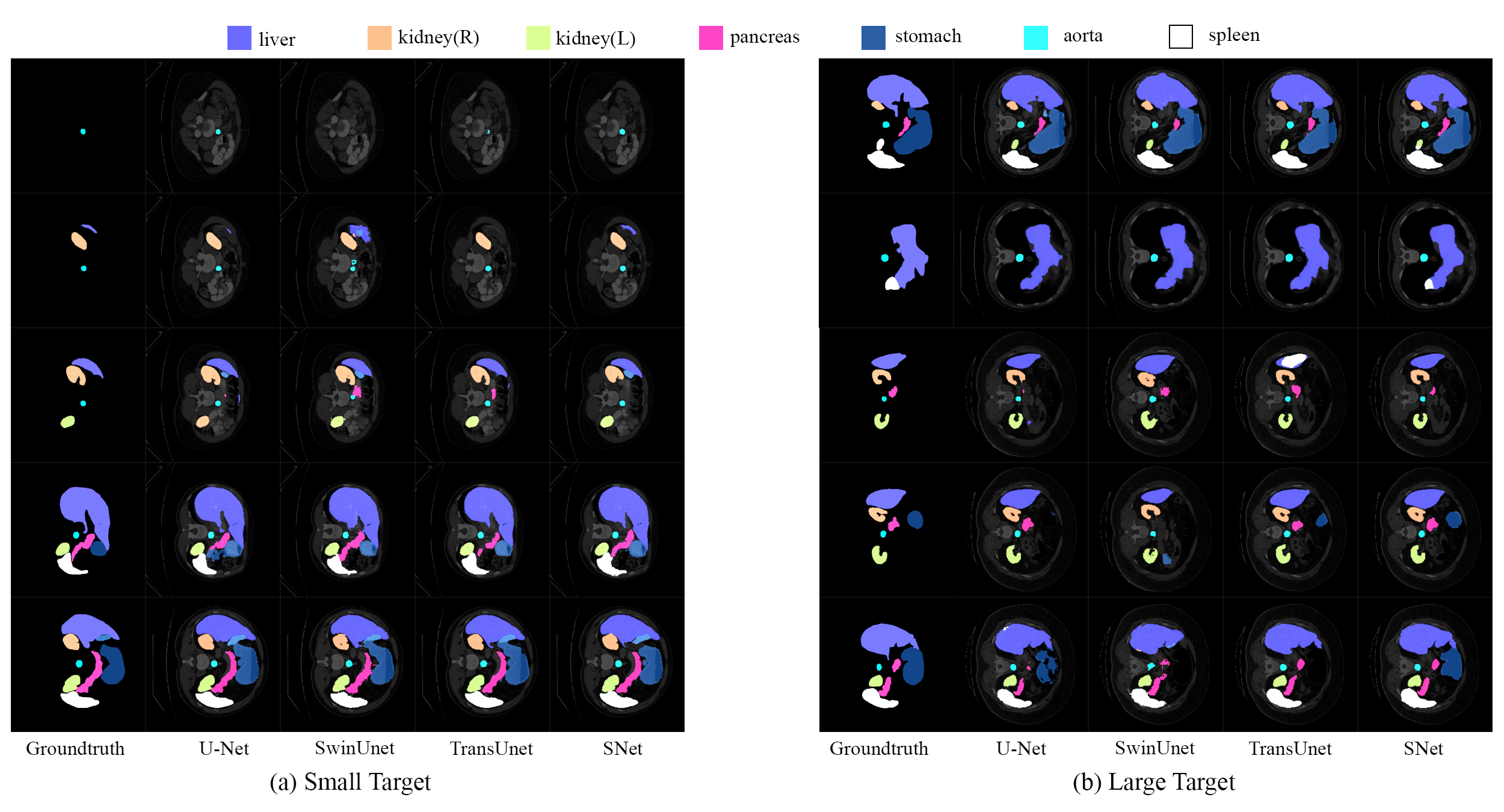}}
\caption{Examples of the prediction results given by SNet, UNet, Swin-UNet, and TransUnet on Synapse validation dataset. The prediction result in the left part shows SNet's superior performance on small targets such as Gallbladder and Pancreas. Meanwhile, it also works well on large targets as illustrated on the right part.}
\label{visual}
\end{figure}

Figure~\ref{visual} visualizes the results presented in Table~\ref{table:1} by the proposed SNet against other competing baselines UNet, SwinUnet, and TransUnet.~\footnote{The visualized models have been trained which yield results equivalent to those reported in the papers. As CASTformers and Missformer have not offered pre-trained models, they are not included in the visualization.}
The result in the left column Figure~\ref{visual} (a) illustrates the superiority of the SNet on small targets such as Gallbladder and Pancreas. Evidently, it offers more accurate predictions without introducing a lot of false detection. Meanwhile, the prediction result in the right column Figure~\ref{visual} (b) illustrates the SNet's excellent performance on large targets such as the Liver. In addition, SNet appears to work well when predicting the detailed parts.

\subsubsection{Results on ACDC Dataset}
The results on the ACDC dataset are summarised in Table~\ref{table:acdc}. We can see that SNet outperforms the other comparative arts by achieving $91.55\%$ dice score, $0.69\%$ dice score better than the highest (Missformer) of previous baselines. 

Similar to its performance on the Synapse dataset, the model shows a noticeable improvement in average dice, enhancing RV and Myo predictions for the two smaller target classes by $1\%$ and $1.57\%$, respectively. However, there is a decline of $1.36\%$ dice score in the LV prediction, a larger target. The results are promising, particularly considering the previously inadequate Myo prediction in all prior models. We have made substantial progress on classes that were challenging to predict, indicating effective handling of information loss. However, due to a trade-off in model performance, there is a decline in the prediction of larger targets. Nevertheless, the overall average dice score demonstrates improvement.

\begin{table}[t!]
\centering
% \tiny
\scriptsize
\caption{Comparison on ACDC data. Bold font indicates the best result, and the second-best results are highlighted underlined. The results of other experiments are original from Missformer. Avg. is the average dice of all the classes.}
% \resizebox{\linewidth}{!}{%
\begin{tabular}{c|cccc}
    \hline
    \multirow{2}{*}{Methods} & \multicolumn{4}{c}{Dice ($\%$) $\uparrow$}  \\
    % \rule{0pt}{0pt}
    \cline{2-5}
          & RV             & Myo            & LV            & Avg. \\ \hline
    R50 U-Net\cite{chen2021transunet}   & 87.10          & 80.63          & 94.92         & 87.55 \\
    R50 Att-UNet\cite{chen2021transunet} & 87.58          & 79.20          & 93.47         & 86.75 \\
    R50 ViT\cite{valanarasu2021medical}       & 86.07          & 81.88          & 94.75         & 87.57 \\
    TransUnet\cite{chen2021transunet}     & 88.86          & 84.53          & 95.73         & 89.71 \\
    SwinUnet\cite{cao2021swin}      & 88.55          & 85.62          & \textbf{95.83} & 90.00 \\ 
    Missformer\cite{huang2021missformer}   &\underline{89.55}    & 88.04          & \underline{94.99}  & \underline{90.86} \\
    \textbf{SNet (ours)}       & \textbf{90.56} & \textbf{89.61} & 94.47         & \textbf{91.55} \\
    \hline
\end{tabular}
% \caption{Comparison on ACDC data. Bold font indicates the best result, and the second-best results are highlighted underlined. The results of other experiments are original from Missformer. Avg. is the average dice of all the classes.}
\label{table:acdc}
\end{table}

\subsubsection{Results on MoNuSeg Dataset}
{As shown in Table~\ref{tab:monuseg_comparison}, the proposed SNet achieved the best results on the
MoNuSeg dataset compared to other models in the table, with IoU and
dice score of 69.6 and 81.9, respectively. Many small targets can be seen in each image in the MoNuSeg dataset. 
SNet outperformed the CT-Net~\cite{ZHANG2024298}, the recent SOTA, with a 1\% increase in IoU and a 0.6\% increase in dice score.}
\begin{table}[ht]
\caption{Comparison of SNet and other advanced methods on the MoNuSeg dataset.}
    \scriptsize
    \centering
    \begin{tabular}{lcc}
        \hline
        Methods                               & IoU ($\%$) $\uparrow$   & Dice ($\%$) $\uparrow$    \\
        \hline
        UNet\cite{ronneberger2015u}         & 59.40  & 74.00  \\
        Att-UNet\cite{oktay2018attention}    & 62.60  & 76.20  \\
        TransUNet\cite{chen2021transunet}           & 65.70  & 79.20  \\
        SwinUNet\cite{cao2021swin}          & 64.70  & 78.50  \\
        UCTransNet-pre\cite{wang2022uctransnet}  & 63.80  & 77.20  \\
        ATTransUNet\cite{li2023attransunet}         & 65.50  & 79.20  \\
        CT-Net\cite{ZHANG2024298}                         & \underline{66.50}  & \underline{79.80}  \\
        \textbf{SNet(ours)}  &\textbf{69.60}&\textbf{81.90}\\
        \hline
    \end{tabular}
    % \caption{Comparison of SNet and other advanced methods on the MoNuSeg dataset.}
    \label{tab:monuseg_comparison}
\end{table}

\begin{table}[]
\label{params}
\caption{Comparison on parameters of different models.}
% \resizebox{\linewidth}{!}{
\begin{tabular}{c|cccccccc}
\hline
Models & U-Net\cite{ronneberger2015u} & Att-UNet\cite{oktay2018attention} & R50 ViT\cite{valanarasu2021medical} & TransUnet\cite{chen2021transunet}  \\
\hline
Params & 7.2M   & 19.8M           & 488.25M & 105.3M     \\
\hline
Models& SwinUnet\cite{cao2021swin} & Missformer\cite{huang2021missformer}&HiFormer\cite{heidari2023hiformer} & SNet (ours) \\
\hline
Params     & 41.4M      & 40.5M     &29.52M  & 38.7M \\
\hline
\end{tabular}
% }
% \caption{Comparison on parameters of different models.}
\end{table}

\subsection{Ablation Study}
To validate the necessity of the stagger fusion, as well as contributing sub-components including FFB, FEB, and GAB, we conduct the following ablation studies by evaluating predictions on the Synapse dataset against scenarios when the upon-mentioned components are disabled.  

\subsubsection{Effect of the Stagger Fusion}  
We verify the effectiveness of stagger fusion by using
a U-shaped model employing unstagger fusion, wherein features were merged from comparable CNN and ViT layers. 
This comparative model also integrated sub-components such as FEB, FFB, and GAB, ensuring an apples-to-apples comparison by keeping all other training settings constant.
The results reported in Table~\ref{stagger} indicate that stagger fusion improves predictions by $4.16\%$ dice score, which thus verifies its effectiveness.

\begin{table}[]
\centering
\scriptsize
\caption{Performance comparison between the stagger fusion and unstagger fusion. Avg. is the average dice of all the classes. The bold values denote the best scores.}
% \resizebox{\linewidth}{!}{%
\begin{tabular}{c|c|ccc|c}
\hline
\multirow{2}{*}{Model} & \multirow{2}{*}{HD $\downarrow$} &  \multicolumn{4}{c}{Dice ($\%$) $\uparrow$} \\
\cline{3-6}
&&  Tissues & Small & Large &Avg.\\ \hline
Unstagger         &   30.11 &89.56 &70.29   &86.54 &  78.79 \\
\textbf{Stagger (ours)}       &  \textbf{15.74} & \textbf{90.19} &\textbf{77.01} &\textbf{88.73} &  \textbf{83.05} \\
\hline
\end{tabular}
% }
\label{stagger}
% \caption{Performance comparison between the stagger fusion and unstagger fusion. Avg. is the average dice of all the classes. The bold values denote the best scores.}
\end{table}

As shown in Figure~\ref{banner} (a), features selected in the stagger fusion are more similar. In contrast, features selected in the unstagger fusion are characterized by diversity with significant semantic gaps.
Stagger fusion provides meaningful information and preserves the characteristics of each feature map. However, unstagger fusion leads to the overshadowing of certain feature characteristics. It means that there will be a large information loss after fusion. This could be detrimental, especially in tasks requiring fine-grained feature discernment. As a result, unstagger fusion models may not be able to segment small-sized targets.

subsubsection{Effect of the FEB}
In Table~\ref{ABFEM}, we present the average results of the backbone with stagger fusion on tissues, small, and large targets against the scenario when FEB is off. It can be observed that the prediction can be improved by 0.32\%, 7.32\%, and $1.91\%$ on dice score on tissue, as well as small and large targets respectively. On average, the dice score can be promoted by 3.68\%, demonstrating that the proposed FEB is useful.
\begin{table}[t!]
\centering
\scriptsize
\caption{Effect of FEB, FFB, and GAB on tissues, small targets, and large targets. Avg. is the average dice of all the classes. The bold values denote the best scores.}
\begin{tabular}{cccc|ccc|c}
\hline
\multirow{2}{*}{Stagger}&\multirow{2}{*}{FEB}&\multirow{2}{*}{FFB}&\multirow{2}{*}{GAB}  &\multicolumn{4}{c}{Dice (\%) $\uparrow$}    \\
\cline{5-8}
&&&&Tissues&Small&Large &Avg. \\ \hline
\Checkmark   &-&-& -    & 87.11  &67.21& 84.22 &  76.81      \\
\Checkmark&\Checkmark&-&  -   & 87.43& 74.53&86.13 & 80.49 \\
\Checkmark&\Checkmark&\Checkmark&    -  &  89.38&74.87&86.32 &80.97  \\
\Checkmark&\Checkmark&\Checkmark&\Checkmark    &\textbf{90.19}&\textbf{77.01}&\textbf{88.73} &  \textbf{83.05}\\
\hline 
\end{tabular}
\label{ABFEM}
% \caption{Effect of FEB, FFB, and GAB on tissues, small targets, and large targets. Avg. is the average dice of all the classes. The bold values denote the best scores.}
\end{table}

\subsubsection{Effect of the FFB}
We evaluate the effectiveness of the proposed novel FFB by adding it to the backbone presented in FEB. As seen from Table~\ref{ABFEM}, we can find that introducing FFB will bring great improvement in tissue predictions, i.e., promoting the dice score from $87.43\%$ to $89.38\%$. 
Improvements on both small and large targets are over $0.2\%$, clearly showing that the novel FFB is of crucial importance.

\subsubsection{Effect of the GAB} 
We further evaluate the effectiveness of the proposed GAB by recovering it into the backbone, becoming the proposed SNet.
Seen from Table~\ref{ABFEM}, predictions on tissues, small, and large targets are promoted by $0.81\%$, $2.14\%$, and $2.41\%$ dice score respectively. Particularly, prediction on large targets achieves the most significant improvement. In summary, sub-components in the proposed SNet including stagger fusion, FEB, FFB, and GAB are all necessary to achieve improvement on various-sized targets in the medical image segmentation field.

\section{Conclusion}
In this paper, we propose the SNet to segment various-sized medical imaging targets. To be specific, we design the Parallel Module to avoid early fusion and thus alleviate information loss at the early stage, the Stagger Module to fusion the similar distribution from CNNs and ViTs to address possible semantic gaps and decrease the information loss,  and the Information Recovery Module to retrieve complementary information. To incorporate with the proposed SNet, we also engage the Feature Enhancement Module, the Feature Fusion Module, and the Global Attention Module to enhance feature representations. We further theoretically prove that  
stagger fusion combining deep CNN and early ViT features will excel superiority compared to the unstagger approach. Extensive experiments have demonstrated the effectiveness of the SNet and the necessity of those sub-components.

\backmatter

\bmhead{Acknowledgements}

The work was partially supported by the following: National Natural Science Foundation of China under No.
92370119, and No. 62376113; Jiangsu Science and Technology Program (Natural Science Foundation
of Jiangsu Province) under No. BE2020006-4;
XJTLU Research Development Funding 20-02-60. Computational resources used in this research are provided by the School of Robotics, XJTLU Entrepreneur College (Taicang), Xi'an Jiaotong-Liverpool University.

%%===========================================================================================%%
%% If you are submitting to one of the Nature Portfolio journals, using the eJP submission   %%
%% system, please include the references within the manuscript file itself. You may do this  %%
%% by copying the reference list from your .bbl file, paste it into the main manuscript .tex %%
%% file, and delete the associated \verb+\bibliography+ commands.                            %%
%%===========================================================================================%%

\bibliography{sn-article}
% \bibliography{sn-article-template/sn-bibliography.bib}% common bib file
%% if required, the content of .bbl file can be included here once bbl is generated
% \input sn-article.bbl

\end{document}